\documentclass{birkjour}
\usepackage{pst-pdf}
\pdfoutput=1

\usepackage{algorithmic}
\usepackage{algorithm}

 \usepackage{amsmath}
 \usepackage{amssymb}
 \usepackage{amsthm}
 \usepackage{mathrsfs}

\def\R		{\mathbb{R}}

\def\O		{\mathcal{O}}
\def\d		{\:\mathrm{d}}

\def\L		{\mathscr{L}}

\renewcommand{\vec}[1]{\boldsymbol{#1}}

\newcommand{\clifford}[1]{\ensuremath{C\kern-0.1em\ell_{#1}}}% #1 stands for the values p,q.
\def\e		{\vec{e}}

\def\A		{\vec{A}}
\def\B		{\vec{B}}
\def\x		{\vec{x}}
\def\y		{\vec{y}}

\def\u		{\vec{u}}
\def\v		{\vec{v}}

 \newtheorem{thm}{Theorem}[section]
 \newtheorem{cor}[thm]{Corollary}
 \newtheorem{lem}[thm]{Lemma}
 
 \theoremstyle{definition}
 \newtheorem{defn}[thm]{Definition}
 \theoremstyle{remark}
 
 \newtheorem*{ex}{Example}
 \numberwithin{equation}{section}

\usepackage[hang]{subfigure}        

\begin{document}

\title[Iterative Geometric Correlation]{Detection of Outer Rotations on 3D-Vector Fields with Iterative Geometric Correlation and its Efficiency}

\author[Bujack]{Roxana Bujack}
\address{%
Universit\"at Leipzig\\
Institut f\"ur Informatik\\
Augustusplatz 10\\
04109 Leipzig, Germany}
\email{bujack@informatik.uni-leipzig.de}

\author[Scheuermann]{Gerik Scheuermann}
\address{%
Universit\"at Leipzig\\
Institut f\"ur Informatik\\
Augustusplatz 10\\
04109 Leipzig, Germany}
\email{scheuermann@informatik.uni-leipzig.de}

\author[Hitzer]{Eckhard Hitzer}
\address{%
College of Liberal Arts, Department of Material Science,\\
International Christian University,\\
181-8585 Tokyo, Japan}
\email{hitzer@icu.ac.jp}

\date{\today}

\begin{abstract}
Correlation is a common technique for the detection of shifts. Its generalization to the multidimensional geometric correlation in Clifford algebras has been proven a useful tool for color image processing, because it additionally contains information about a rotational misalignment. But so far the exact correction of a three-dimensional outer rotation could only be achieved in certain special cases. 
\par
In this paper we prove that applying the geometric correlation iteratively has the potential to detect the outer rotational misalignment for arbitrary three-dimensional vector fields.
\par
We further present the explicit iterative algorithm, analyze its efficiency detecting the rotational misalignment in the color space of a color image. The experiments suggest a method for the acceleration of the algorithm, which is practically tested with great success. 
\end{abstract}

\maketitle

\section{Introduction}
%---------------------------------------------------------------------------------------------------------------------------------
In signal processing correlation is a basic technique to determine the similarity or dissimilarity of two signals. It is widely used for image registration, pattern matching, and feature extraction \cite{BRO92,ZIT03}. The idea of using correlation for registration of shifted signals is that at the very position, where the signals match, the correlation function will 
take its maximum, compare \cite{RK82}.
\par 
For a long time the generalization of this method to multidimensional signals has only been an amount of single channel processes. The elements of %geometric or 
Clifford algebras $\clifford{p,q}$, compare \cite{C1878,HS84}, have a natural geometric interpretation, so the analysis of multivariate signals expressed as multivector valued functions is a very reasonable approach.
\par
Scheuermann \cite{Sch99} used Clifford algebras for vector field analysis. Together with Ebling \cite{Ebl03,Ebl06} they developed a pattern matching algorithm based on geometric convolution and correlation and accelerated it by means of a Clifford Fourier transform and its convolution theorem.
\par
At about the same time Sangwine et al. \cite{MSE01} introduced a generalized hypercomplex correlation for quaternions. Together with Ell and Moxey \cite{MSE02,MSE03} they used it to represent color images, interpreted as vector fields, geometrically. They discovered that this geometric correlation not only contains the translational difference of images given by the position of the magnitude peak, but also information about a possible rotational misalignment of two signals and showed how to apply them to approximately correct color space distortions. 
\par
Even though lately other approaches to work with color images were made \cite{GSH10,Schl11,MSM11} we want to extend the work and ideas of Moxey, Ell and Sangwine using hypercomplex correlation. We analyze vector fields $\v(\x):\R^m\to\R^3\subset\clifford{3,0}$ with values interpreted as elements of the geometric algebra $\clifford{3,0}$ and their copies produced from outer rotations. A great advantage of the geometric algebra is that many statements generally hold not just for vectors but for all multivectors. We will make use of that and state the more general formulae, whenever possible. 
\par
\begin{figure}[ht]
\centering
\subfigure[Original vector \newline field:  $\v(\x)$]{\includegraphics[width=0.227\textwidth]{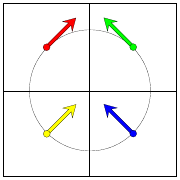}}\quad
\subfigure[Inner rotation: \newline  $\v(\operatorname{R }_{-\alpha}(\x))$]{\includegraphics[width=0.227\textwidth]{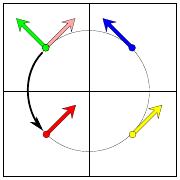}}\quad
\subfigure[Outer rotation:\newline $\operatorname{R}(\v(\x))$]{\includegraphics[width=0.227\textwidth]{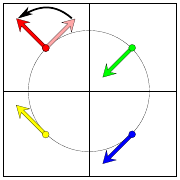}}\quad
\subfigure[Total rotation: \newline $\operatorname{R }_{\alpha}(\v(\operatorname{R }_{-\alpha}(\x)))$]{\includegraphics[width=0.227\textwidth]{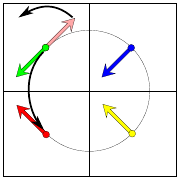}}
\caption{Effect of the rotation operator $\operatorname{R} _{\alpha}$ applied to an example vector field in different ways.\label{f:1}} 
\end{figure}
The term rotational misalignment with respect to multivector fields is ambiguous. In general three types of rotations can be distinguished, compare Figure \ref{f:1}. Let $\operatorname{R} _{\vec P,\alpha}$ be an operator, that describes a mathematically positive rotation by the angle $\alpha\in[0,\pi]$\footnote{As in \cite{MSE02} we encode the sign in the bivector $\vec P$ and deal with positive angles only.} in the plane $P$, spanned by the unit bivector $\vec P$. We say two multivector fields $\A(\x),\B(\x):\R^m\to\clifford{3,0}$ differ by an \textbf{outer rotation} if they suffice
\begin{equation}
\begin{aligned}\label{outer}
 \A(\x)=\operatorname{R} _{\vec P,\alpha}(\B(\x)).
\end{aligned}
\end{equation}
Independent from their position $\x$ the multivector $\A(\x)$ is the rotated copy of the multivector $\B(\x)$. This kind of rotation appears for example in color images, when the color vector space is turned, but the picture is not moved, compare \cite{MSE03}. In the theory of differential geometry this kind of rotation is referred to as a rotation in the tangential space of a manifold. The property of the argument space to be a vector field is not necessary to define an outer rotation. It can as well be a submanifold of a vector field. Since this paper works with outer rotations, please note that all results also hold for this generalization. It can, for example, detect the color space misalignment of a colored sphere like a globe, which can be of interest for the geoscientific domain of satellite images covering the earth.
\par
In contrast to that for $m\leq3$ an \textbf{inner rotation} is described by 
\begin{equation}
\begin{aligned}\A(\x)=\B(\operatorname{R} _{\vec P,-\alpha}(\x)).
\end{aligned}
\end{equation}
Here the starting position of every vector is rotated by $-\alpha$ then the old vector is reattached at the new position. It still points into the old direction. The inner rotation is suitable to describe the rotation of a color image. The color does not change when the picture is turned. 
\par
In the case of bijective fields $\A(\x),\B(\x):\R^3\to\R^3\subset\clifford{3,0}$ a \textbf{total rotation} is a combination of the previous ones defined by 
\begin{equation}
\begin{aligned}\A(\x)=\operatorname{R} _{\vec P,\alpha}(\B(\operatorname{R} _{\vec P,-\alpha}(\x))).
\end{aligned}
\end{equation} 
It can be interpreted as coordinate transform, that means as looking at the multivector field from another point of view. The positions and the multivectors are stiffly connected during the rotation.
\par
With respect to the definition of the correlation there are different formulae in current literature, \cite{Ebl06,MSE03}. We prefer the following one because it satisfies a geometric generalization of the Wiener-Khinchin theorem and because it coincides with the definition of the standard cross-correlation for complex functions in the special case of $\clifford{0,1}$, \cite{RK82}. For vector fields they mostly coincide anyways because of $\overline{\v(\x)}=\v(\x)$, where the overbar denotes reversion.
\begin{defn}
 The \textbf{geometric cross correlation} of two multivector valued functions $\A(\x),\B(\x):\R^m\to\clifford{p,q}$ is a multivector valued function defined by
\begin{equation}
\begin{aligned}
(\A\star \B)(\x):=&\int_{\R^m}\overline{\A(\y)}\B(\y+\x)\d^m\y,
\end{aligned}
\end{equation} 
where $\overline{\A(\y)}$ denotes the reversion $\sum\limits_{k=0}^n(-1)^{\frac12k(k-1)}\langle \A(\y)\rangle_k$.
\end{defn}
To simplify the notation we will only analyze the correlation at the origin. If the vector fields also differ by an inner shift this can for example be detected by evaluating the magnitude of the correlation \cite{Hes86} or phase correlation of the field magnitudes \cite{KH75}. 
%, or from the shift theorem of a suitable geometric Fourier transform \cite{BSH11}. 
Our methods can then be applied analogously to that translated position.
%----------------------------------------------------------------------------------------------------------------------------
\section{Motivation}
%----------------------------------------------------------------------------------------------------------------------------
In two dimensions a mathematically positive\footnote{anticlockwise} outer rotation of a vector field $\R^2\to\R^2\subset\clifford{2,0}$ by the angle $\alpha$ takes the shape
\begin{equation}
\begin{aligned}
\operatorname{R} _{\e_{12},\alpha}(\v(\x))=e^{-\alpha \e_{12}}\v(\x).
\end{aligned} 
\end{equation}
So the product of the vector field and its copy at any position $\x\in\R^m$ yields
\begin{equation}
\begin{aligned}\label{outerprod}
\operatorname{R} _{\e_{12},\alpha}(\v(\x))\v(\x)=e^{-\alpha \e_{12}}\v(\x)\v(\x)=\v(\x)^2e^{-\alpha \e_{12}},
\end{aligned} 
\end{equation}
with $\v(\x)^2=\v(\x)\overline{\v(\x)}=||\v(\x)||_2^2\in\R$ and the rotation can fully be restored by rotating back with the inverse of (\ref{outerprod}) or explicitly calculating $\alpha$ as described in \cite{Hes86}. This property is inherited by the geometric correlation at the origin
\begin{equation}
\begin{aligned}
(\operatorname{R} _{\e_{12},\alpha}(\v)\star \v)(0)&=\int_{\R^m}\overline{\operatorname{R} _{\e_{12},\alpha}(\v(\x))}\v(\x)\d^m\x
=||\v(\x)||_{L^2}^2e^{-\alpha \e_{12}},
\end{aligned} 
\end{equation}
which is to be preferred because of its robustness.
\par
%-----------------------------------------------------------------------------------------------------------------------------------
%\section{Three-dimensional geometric product}
%-----------------------------------------------------------------------------------------------------------------------------------
In three dimensions not only the angle but also the plane of rotation $P$ has to be detected in order to reconstruct the whole transform. We want to analyze if the geometric correlation at the origin contains enough information here, too. First we look at two vectors $\u,\v$ that suffice $\u=\operatorname{R} _{\vec P,\alpha}(\v)$. Their geometric product 
\begin{equation}
\begin{aligned}\label{vectors}
\u\v=&\u\cdot \v+\u\wedge \v
=|\u||\v|\big(\cos(\angle(\u,\v))+\sin(\angle(\u,\v))\frac{\u\wedge \v}{|\u\wedge \v|}\big)
\end{aligned}
\end{equation}
contains an angle and a plane and therefore seems very motivating. But the rotation $\operatorname{R} _{\vec P,\alpha}$ we used is not necessarily the one which is described by $\operatorname{R}_{\frac{\u\wedge \v}{|\u\wedge \v|},\angle(\u,\v)}$. These two rotations only coincide if the vectors lie completely within the plane $P$. The reason for that is as follows. A vector and its rotated copy do not contain enough information to reconstruct the rotation that produced the copy. Figure \ref{f:2} shows some of the infinitely many different rotations that can result in the same copy. Regard the set of all circles $C$, that contain the end points of $\u$ and $\v$ and are located on the sphere $S_{|\u|}(0)$ with radius $r=|\u|=|\v|$ centered at the origin. Every plane that includes a circle in $C$ is a possible plane for the rotation from $\v$ to $\u$.
The information we get out of the geometric product belongs to the plane that fully contains both vectors. This rotation is the one that has the smallest angle of all possible ones and forms the largest circle on a great circle of the sphere, shown on the very left in Figure \ref{f:2}.
\par
\begin{figure}[ht]
\centering
\subfigure[Rotation with the smallest angle]{\includegraphics[width=0.227\textwidth]{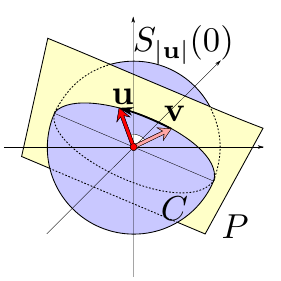}}\quad
\subfigure[Rotation with a rather small angle]{\includegraphics[width=0.227\textwidth]{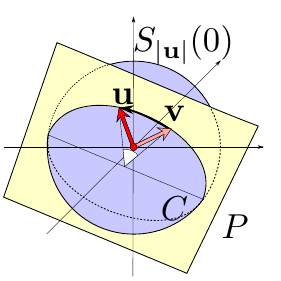}}\quad
\subfigure[Rotation with a rather larger angle]{\includegraphics[width=0.227\textwidth]{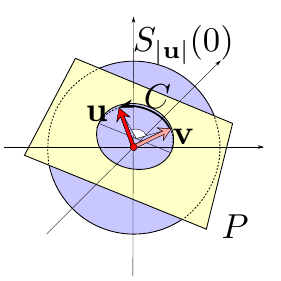}}\quad
\subfigure[Rotation with the largest angle]{\includegraphics[width=0.227\textwidth]{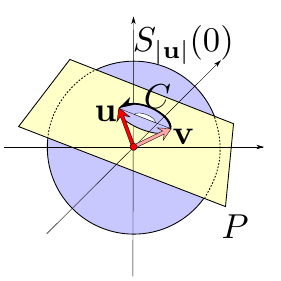}}
\caption{Different rotations of a vector $\u$ that lead to the same result $\v$.\label{f:2}} 
\end{figure}
% \begin{figure}[ht]
% \begin{minipage}{0.23\textwidth}
% \centering
% \psset{unit=0.55pt}
%   \input{Rdrehung1}
% \end{minipage}
% \hspace{0.1cm}
% \begin{minipage}{0.23\textwidth}
% \centering
% \psset{unit=0.55pt}
%   \input{Rdrehung2}
% \end{minipage}
% \hspace{0.1cm}
% \begin{minipage}{0.23\textwidth}
% \centering
% \psset{unit=0.55pt}
%   \input{Rdrehung3}
% \end{minipage}
% \hspace{0.1cm}
% \begin{minipage}{0.23\textwidth}
% \centering
% \psset{unit=0.55pt}
%   \input{Rdrehung4}
% \end{minipage}
% \caption{Different rotations of a vector $\u$ that lead to the same result $\v$.}\label{f:2}
% \end{figure}
For the product of just two vectors the information from the geometric product is sufficient to realign them, but for a whole vector field the detected rotation from the correlation will in general not be the correct one. Moxey et al. already stated in \cite{MSE03} that the hypercomplex correlation can effectively compute the rotation over two images, but that the perfect mapping can only be found, if specific conditions hold, for example if the images consist of one color only.
\begin{figure}[ht]
\centering
\subfigure[Vector field from (\ref{bsp1a})]{\includegraphics[width=0.48\textwidth]{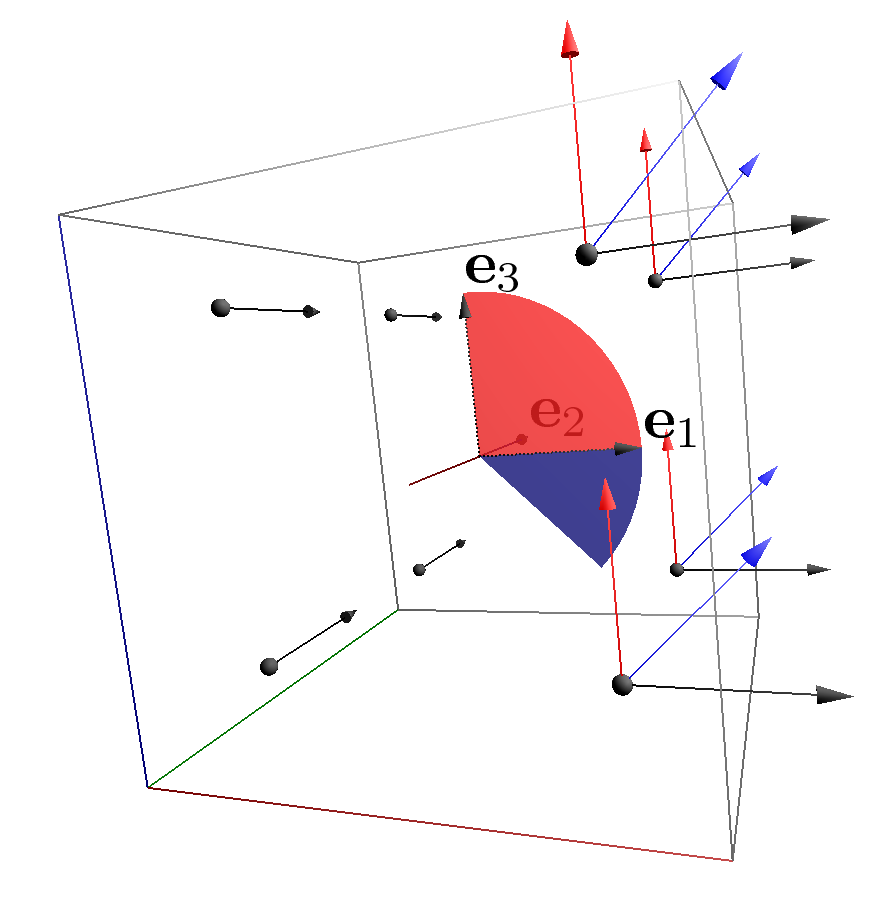}}\quad
\subfigure[Vector field from (\ref{bsp2a})]{\includegraphics[width=0.48\textwidth]{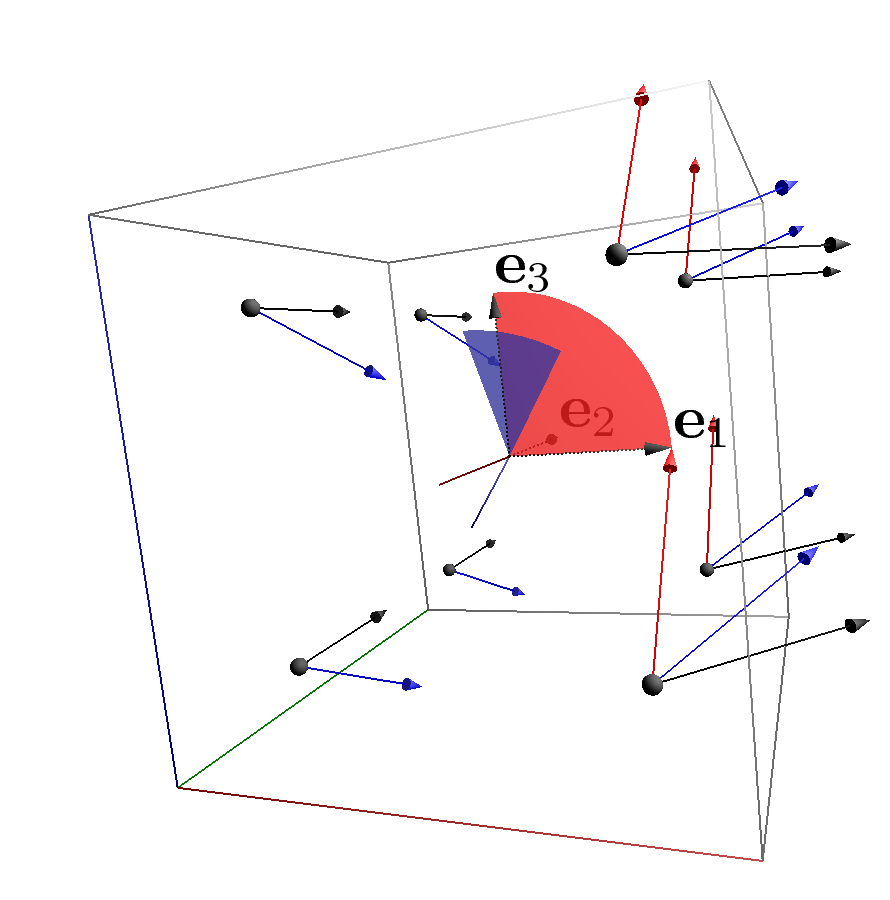}}\quad
\caption{Visualization of the vector fields from the example using CLUCalc \cite{Per09}. Original fields are depicted in black, rotated copies in red, corrected fields after application of the correlation rotor in blue.\label{f:3}} 
\end{figure}
% \begin{figure}[ht]
% \begin{minipage}{0.48\textwidth}
% \centering
% %\psset{unit=1pt}
%   \includegraphics[width=5.5cm,height=5.5cm]{bsp1w}
% \end{minipage}
% \hspace{0.2cm}
% \begin{minipage}{0.48\textwidth}
% \centering
% %\psset{unit=1pt}
%   \includegraphics[width=5.5cm,height=5.5cm]{bsp2w}
% \end{minipage}
% \caption{Left: vector field (\ref{bsp1a}). Right: vector field (\ref{bsp2a}). Original fields are depicted in black, rotated copies in red, corrected fields after application of the correlation rotor in blue using CLUCalc \cite{Per09}.}\label{f:3}
% \end{figure}
%
\begin{ex}
The geometric correlation of the vector field
\begin{equation}
\begin{aligned}\label{bsp1a}
\v(\x)=&\begin{cases}
         \e_1,&\text{ for }x_1,x_2,x_3\in(-1,1), x_1\geq 0,\\
\e_2,&\text{ for }x_1,x_2,x_3\in(-1,1), x_1<0,\\
0,&\text{ else,}\\
        \end{cases}
\end{aligned}
\end{equation}
and its copy rotated by $\operatorname{R}_{\e_{13},\frac\pi2}$, which are shown on the left of Figure \ref{f:3}, suffices
\begin{equation}
\begin{aligned}
(\operatorname{R}_{\e_{13},\frac\pi2}(\v)\star\v)(0)
=&\int_{-1}^1\int_{-1}^1( \int_{0}^1\e_3\e_1 \d x_1
+\int_{-1}^0 \e_2\e_2\d x_1)\d x_3\d x_2\\
=&-4\e_{13}+4
\\=&\sqrt{32}e^{-\frac\pi4\e_{13}}.
\end{aligned}
\end{equation}
Here the unit bivector $\e_{13}$ indeed describes the rotational plane we looked for, but the angle $\frac\pi4$ is only half the angle of the original rotation. The result is that the restored field and the original one do not match.
\par
In general the correlation even detects a wrong rotational plane, consider for example
\begin{equation}
\begin{aligned}
\v(\x)=&\begin{cases}\label{bsp2a}
         \e_1+\e_2,&\text{ for }x_1,x_2,x_3\in(-1,1), x_1\geq 0,\\
\e_2,&\text{ for }x_1,x_2,x_3\in(-1,1), x_1<0,\\
0,&\text{ else,}\\
        \end{cases}
\end{aligned}
\end{equation}
rotated by $\operatorname{R}_{\e_{13},\frac\pi2}$, depicted on the right of Figure \ref{f:3}. The geometric correlation will return the plane spanned by $\e_{12}+\e_{13}+\e_{23}$ and the angle $\arctan(\frac{\sqrt{3}}{2})$, which are both incorrect.
\end{ex}
In the following section we will analyze how effective the correlation can calculate the rotation and prove, that despite the impression, the previous example gives, the geometric correlation contains enough information to reconstruct the misalignment.
%----------------------------------------------------------------------------------------------------------------------------
\section{Outer Rotation and the Geometric Correlation}
%----------------------------------------------------------------------------------------------------------------------------
The three-dimensional mathematically positive outer rotation (\ref{outer}) of a vector field by the angle $\alpha\in[0,\pi]$ along the plane $P$, spanned by the unit bivector $\vec P$, takes the shape
\begin{equation}
\begin{aligned}
\operatorname{R} _{P,\alpha}(\v(\x))=&e^{-\frac\alpha2 \vec P}\v(\x)e^{\frac\alpha2 \vec P}
=e^{-\alpha \vec P}\v_{\parallel\vec P}(\x)+\v_{\perp\vec P}(\x).
\end{aligned} 
\end{equation}
So the product of the vector field and its copy at any position
\begin{equation}
\begin{aligned}\label{outer_prod}
\operatorname{R} _{P,\alpha}(\v(\x))\v(\x)=&
e^{-\alpha \vec P}\v_{\parallel\vec P}(\x)^2+(e^{-\alpha \vec P}-1)\v_{\parallel\vec P}(\x)\v_{\perp\vec P}(\x)+\v_{\perp\vec P}(\x)^2
\end{aligned} 
\end{equation}
does usually not simply yield the rotation we looked for, like in the two-dimensional case, but a rather good approximation depending on the parallel and the orthogonal parts of the vector field with respect to the plane of rotation. In order to keep the notation short we partly drop the argument $\x$ of the vector fields $\v(\x)$ by just writing $\v$ and assume without loss of generality
\begin{equation}
\begin{aligned}\label{assumption}
||\operatorname{R} _{P,\alpha}(\v)||_{L^2}^2=||\v||_{L^2}^2=||\v_{\parallel\vec P}||_{L^2}^2+||\v_{\perp\vec P}||_{L^2}^2=1.
\end{aligned} 
\end{equation}
\begin{lem}\label{l:outer_prod}
Let $\v\in L^2(\R^m,\R^{3,0}\subset\clifford{3,0})$ be a square integrable vector field and $\operatorname{R} _{P,\alpha}(\v)$ its copy from an outer rotation. The rotational misalignment of $\operatorname{R} _{P,\alpha}(\v)$ does not increase if we apply the outer rotation encoded in the normalized geometric cross correlation 
\begin{equation}
\begin{aligned}
\frac{(\operatorname{R} _{P,\alpha}(\v)\star\v)(0)}{{||\operatorname{R} _{P,\alpha}(\v)||_{L^2}||\v||_{L^2}}}.
\end{aligned} 
\end{equation}
\end{lem}
\begin{proof}
%  The product of $\operatorname{R} _{P,\alpha}(\v(\x))$ and $\v(\x)$ suffices (\ref{outer_prod}), so the cross correlation at the origin yields
% \begin{equation}
% \begin{aligned}
% (\operatorname{R} _{P,\alpha}(\v)\star\v)(0)=&
% e^{-\alpha \vec P}\int\v_{\parallel\vec P}(\x)^2\d \x+(e^{-\alpha \vec P}-1)\int\v_{\parallel\vec P}(\x)\v_{\perp\vec P}(\x)\d \x
% \\&+\int\v_{\perp\vec P}(\x)^2\d \x
% \end{aligned} 
% \end{equation}
We denote the polar form of the normalized geometric cross correlation by $e^{\varphi\vec Q}$ with the unit bivector $\vec Q$ and $\varphi\in[0,\pi]$. So using (\ref{outer_prod}) and the assumption (\ref{assumption}) we get
\begin{equation}
\begin{aligned}
e^{\varphi\vec Q}
=&(\operatorname{R} _{P,\alpha}(\v)\star\v)(0)
\\=&
\int e^{-\alpha \vec P}\v_{\parallel\vec P}(\x)^2+(e^{-\alpha \vec P}-1)\v_{\parallel\vec P}(\x)\v_{\perp\vec P}(\x)
+\v_{\perp\vec P}(\x)^2\d^m \x
\\=&e^{-\alpha \vec P}||\v_{\parallel\vec P}||_{L^2}^2+(e^{-\alpha \vec P}-1)\int\v_{\parallel\vec P}\v_{\perp\vec P}\d \x
+||\v_{\perp\vec P}||_{L^2}^2
\end{aligned} 
\end{equation}
with the scalar part
\begin{equation}
\begin{aligned}\label{scphiq}
\cos(\varphi)=&\langle e^{\varphi\vec Q}\rangle_0
=\cos(\alpha)||\v_{\parallel\vec P}||_{L^2}^2+||\v_{\perp\vec P}||_{L^2}^2
\end{aligned} 
\end{equation}
and the bivector part
\begin{equation}
\begin{aligned}\label{bivphiq}
\sin(\varphi)\vec Q&=\langle e^{\varphi\vec Q}\rangle_2
\\&=-\sin(\alpha) \vec P||\v_{\parallel\vec P}||_{L^2}^2+(-\sin(\alpha) \vec P+\cos(\alpha)-1)\int\v_{\parallel\vec P}\v_{\perp\vec P}\d \x
\end{aligned} 
\end{equation}
with squared magnitude
\begin{equation}
\begin{aligned}
||\langle e^{\varphi\vec Q}\rangle_2||^2
=&\sin(\alpha)^2||\v_{\parallel\vec P}||_{L^2}^4+(2-2\cos(\alpha))\,||\int\v_{\parallel\vec P}\v_{\perp\vec P}\d \x||^2.
\end{aligned} 
\end{equation}
Thats why we know the explicit expressions for
\begin{equation}
\begin{aligned}\label{bivphiq2}
\vec Q
=&\frac{\langle e^{\varphi\vec Q}\rangle_2}{||\langle e^{\varphi\vec Q}\rangle_2||}
\end{aligned} 
\end{equation}
and
\begin{equation}
\begin{aligned}
\varphi
=&\operatorname{atan2}(||\langle e^{\varphi\vec Q}\rangle_2||,\langle e^{\varphi\vec Q}\rangle_0).
\end{aligned} 
\end{equation}
The outer rotation encoded in the correlation applied to $\operatorname{R} _{P,\alpha}(\v)$ takes the shape
\begin{equation}
\begin{aligned}
e^{-\frac\varphi2\vec Q}\operatorname{R} _{P,\alpha}(\v)e^{\frac\varphi2\vec Q}=&e^{-\frac\varphi2\vec Q}e^{-\frac\alpha2\vec P}\v e^{\frac\alpha2\vec P}e^{\frac\varphi2\vec Q}.
\end{aligned} 
\end{equation}
The composition of the two rotations is a rotation itself. It shall be written as 
\begin{equation}
\begin{aligned}
e^{-\frac\beta2\vec R}\v e^{\frac\beta2\vec R}
\end{aligned} 
\end{equation}
with the unit bivector $\vec R$ and angle $\beta\in[0,\pi]$, so we get the relation
\begin{equation}
\begin{aligned}\label{beta_def}
e^{\frac\beta2\vec R}=e^{\frac\alpha2\vec P}e^{\frac\varphi2\vec Q}.
\end{aligned} 
\end{equation}
To determine whether or not it is smaller than the original one, it is sufficient to compare $\beta$ and $\alpha$. We will prove that $\beta\leq\alpha$ by proving the inequality
\begin{equation}
\begin{aligned}\label{beta<alpha}
\frac\beta2=\arg(e^{\frac\beta2\vec R})=\arg(e^{\frac\alpha2\vec P}e^{\frac\varphi2\vec Q})\leq\arg(e^{\frac\alpha2\vec P})=\frac\alpha2.
\end{aligned} 
\end{equation}
We evaluate (\ref{beta_def}) by inserting   
%$\vec Q =\langle e^{\varphi\vec Q}\rangle_2||\langle e^{\varphi\vec Q}\rangle_2||^{-1}$,
(\ref{bivphiq2}) and (\ref{bivphiq}) %simplified by means of addition theorems 
and get
\begin{equation}
\begin{aligned}\label{beta3}
e^{\frac\beta2\vec R}
=&e^{\frac\alpha2\vec P}e^{\frac\varphi2\vec Q}
\\=&\cos(\frac\alpha2)\cos(\frac\varphi2)+\cos(\frac\alpha2)\sin(\frac\varphi2)\vec Q+\sin(\frac\alpha2)\cos(\frac\varphi2)\vec P
\\&+\sin(\frac\alpha2)\sin(\frac\varphi2)\vec P\vec Q\\
=&\cos(\frac\alpha2)\cos(\frac\varphi2)+\sin(\frac\alpha2)\cos(\frac\varphi2)\vec P
\\&+\frac{\sin(\frac\varphi2)}{||\langle e^{\varphi\vec Q}\rangle_2||}\big(-\cos(\frac\alpha2)\sin(\alpha) \vec P||\v_{\parallel\vec P}||_{L^2}^2
\\&+\cos(\frac\alpha2)(\cos(\alpha)-1)\int\v_{\parallel\vec P}\v_{\perp\vec P}\d \x
\\&-\cos(\frac\alpha2)\sin(\alpha) \vec P\int\v_{\parallel\vec P}\v_{\perp\vec P}\d \x
+\sin(\frac\alpha2)\sin(\alpha) \int\v_{\parallel\vec P}\v_{\perp\vec P}\d \x
\\&+\sin(\frac\alpha2)\sin(\alpha)||\v_{\parallel\vec P}||_{L^2}^2
+\sin(\frac\alpha2)(\cos(\alpha)-1)\vec P\int\v_{\parallel\vec P}\v_{\perp\vec P}\d \x
\big),
\end{aligned} 
\end{equation}
and applying addition theorems 
on the $\int\v_{\parallel\vec P}\v_{\perp\vec P}\d \x$-parts 
leads to
\begin{equation}
\begin{aligned}\label{beta4}
&\cos(\frac\alpha2)(\cos(\alpha)-1)\int\v_{\parallel\vec P}\v_{\perp\vec P}\d \x+\sin(\frac\alpha2)\sin(\alpha)\int\v_{\parallel\vec P}\v_{\perp\vec P}\d \x 
\\&=\big(\cos(\frac\alpha2)\cos(\alpha)+\sin(\frac\alpha2)\sin(\alpha)-\cos(\frac\alpha2)\big)\int\v_{\parallel\vec P}\v_{\perp\vec P}\d \x
\\&=\big(\cos(\frac\alpha2-\alpha)-\cos(\frac\alpha2)\big)\int\v_{\parallel\vec P}\v_{\perp\vec P}\d \x
\\&=0,
\end{aligned} 
\end{equation}
and on the $\vec P\int\v_{\parallel\vec P}\v_{\perp\vec P}\d \x$-parts to
\begin{equation}
\begin{aligned}\label{beta5}
&\sin(\frac\alpha2)(\cos(\alpha)-1)\vec P\int\v_{\parallel\vec P}\v_{\perp\vec P}\d \x
-\cos(\frac\alpha2)\sin(\alpha) \vec P\int\v_{\parallel\vec P}\v_{\perp\vec P}\d \x
\\&=\big(\sin(\frac\alpha2)\cos(\alpha)-\cos(\frac\alpha2)\sin(\alpha)-\sin(\frac\alpha2)\big)\vec P\int\v_{\parallel\vec P}\v_{\perp\vec P}\d \x
\\&=\big(\sin(\frac\alpha2-\alpha)-\sin(\frac\alpha2)\big)\vec P\int\v_{\parallel\vec P}\v_{\perp\vec P}\d \x
\\&=-2\sin(\frac\alpha2)\vec P\int\v_{\parallel\vec P}\v_{\perp\vec P}\d \x.
\end{aligned} 
\end{equation}
We insert (\ref{beta4}) and (\ref{beta5}) in (\ref{beta3}) and get
\begin{equation}
\begin{aligned}
e^{\frac\beta2\vec R}
=&\cos(\frac\alpha2)\cos(\frac\varphi2)+\sin(\frac\alpha2)\cos(\frac\varphi2)\vec P
\\&+\frac{\sin(\frac\varphi2)}{||\langle e^{\varphi\vec Q}\rangle_2||}\big(-\cos(\frac\alpha2)\sin(\alpha) \vec P||\v_{\parallel\vec P}||_{L^2}^2
\\&+\sin(\frac\alpha2)\sin(\alpha)||\v_{\parallel\vec P}||_{L^2}^2
-2\sin(\frac\alpha2)\vec P\int\v_{\parallel\vec P}\v_{\perp\vec P}\d \x\big).
\end{aligned} 
\end{equation}
% The direction of this new misalignment is of no interest for us. We want to analyze the magnitude of it from
% \begin{equation}
% \begin{aligned}
% \frac\beta2=&\operatorname{atan2}(||\langle e^{\frac\beta2\vec R}\rangle_2||,\langle e^{\frac\beta2\vec R}\rangle_0).
% \end{aligned} 
% \end{equation}
Its scalar part
\begin{equation}
\begin{aligned}\label{sc}
\langle e^{\frac\beta2\vec R}\rangle_0
=&\cos(\frac\alpha2)\cos(\frac\varphi2)
+\frac{1}{||\langle e^{\varphi\vec Q}\rangle_2||}\sin(\frac\alpha2)\sin(\frac\varphi2)\sin(\alpha)||\v_{\parallel\vec P}||_{L^2}^2
\end{aligned} 
\end{equation}
is generally positive, because $\alpha,\varphi\in[0,\pi]$
%. So the $\operatorname{atan2}$ is the arctangent. The 
and the bivector part 
\begin{equation}
\begin{aligned}
\langle e^{\frac\beta2\vec R}\rangle_2
=&\sin(\frac\alpha2)\cos(\frac\varphi2)\vec P-\frac{1}{||\langle e^{\varphi\vec Q}\rangle_2||}
\cos(\frac\alpha2)\sin(\frac\varphi2)\sin(\alpha)||\v_{\parallel\vec P}||_{L^2}^2\vec P
\\&-\frac{2}{||\langle e^{\varphi\vec Q}\rangle_2||}\sin(\frac\alpha2)\sin(\frac\varphi2)\vec P\int\v_{\parallel\vec P}\v_{\perp\vec P}\d \x
\end{aligned} 
\end{equation}
has the squared norm
\begin{equation}
\begin{aligned}\label{bivnorm}
||\langle e^{\frac\beta2\vec R}\rangle_2||^2
=&\sin(\frac\alpha2)^2\cos(\frac\varphi2)^2
-\frac{1}{||\langle e^{\varphi\vec Q}\rangle_2||}\cos(\frac\varphi2)\sin(\frac\varphi2)\sin(\alpha)^2||\v_{\parallel\vec P}||_{L^2}^2
\\&+\frac{1}{||\langle e^{\varphi\vec Q}\rangle_2||^2}\cos(\frac\alpha2)^2\sin(\frac\varphi2)^2\sin(\alpha)^2||\v_{\parallel\vec P}||_{L^2}^4
\\&+\frac{4}{||\langle e^{\varphi\vec Q}\rangle_2||^2}\sin(\frac\alpha2)^2\sin(\frac\varphi2)^2||\int\v_{\parallel\vec P}\v_{\perp\vec P}\d \x||^2.
\end{aligned} 
\end{equation}
For the next inequalities we use, that all appearing parts are positive and that the tangent and the quadratic function are monotonically increasing for positive arguments. We get
\begin{equation}
\begin{aligned}
\frac\beta2\leq\frac\alpha2
&\Leftrightarrow
\arctan(\frac{||\langle e^{\frac\beta2\vec R}\rangle_2||}{\langle e^{\frac\beta2\vec R}\rangle_0})
\leq\arctan(\frac{\sin(\frac\alpha2)}{\cos(\frac\alpha2)})
\\&\Leftrightarrow
\frac{||\langle e^{\frac\beta2\vec R}\rangle_2||}{\langle e^{\frac\beta2\vec R}\rangle_0}
\leq\frac{\sin(\frac\alpha2)}{\cos(\frac\alpha2)}
% \\&\Leftrightarrow
% ||\langle e^{\frac\beta2\vec R}\rangle_2||\cos(\frac\alpha2)
% \leq\sin(\frac\alpha2)\langle e^{\frac\beta2\vec R}\rangle_0
% \\&\Leftrightarrow
% ||\langle e^{\frac\beta2\vec R}\rangle_2||^2\cos(\frac\alpha2)^2
% \leq\sin(\frac\alpha2)^2\langle e^{\frac\beta2\vec R}\rangle_0^2
\\&\Leftrightarrow
||\langle e^{\frac\beta2\vec R}\rangle_2||^2\cos(\frac\alpha2)^2||\langle e^{\varphi\vec Q}\rangle_2||^2
\leq\sin(\frac\alpha2)^2\langle e^{\frac\beta2\vec R}\rangle_0^2||\langle e^{\varphi\vec Q}\rangle_2||^2.
\end{aligned} 
\end{equation}
Now we insert the scalar part (\ref{sc}) and the bivector norm (\ref{bivnorm})
\begin{equation}
\begin{aligned}
&\sin(\frac\alpha2)^2\cos(\frac\alpha2)^2\cos(\frac\varphi2)^2||\langle e^{\varphi\vec Q}\rangle_2||^2
+\sin(\frac\varphi2)^2\cos(\frac\alpha2)^4\sin(\alpha)^2||\v_{\parallel\vec P}||_{L^2}^4
\\&-||\langle e^{\varphi\vec Q}\rangle_2|| \cos(\frac\varphi2)\sin(\frac\varphi2)\cos(\frac\alpha2)^2\sin(\alpha)^2||\v_{\parallel\vec P}||_{L^2}^2
\\&+4\sin(\frac\varphi2)^2\sin(\frac\alpha2)^2\cos(\frac\alpha2)^2||\int\v_{\parallel\vec P}\v_{\perp\vec P}\d \x||^2
\\\leq
&\sin(\frac\alpha2)^2\cos(\frac\alpha2)^2\cos(\frac\varphi2)^2||\langle e^{\varphi\vec Q}\rangle_2||^2
+\sin(\frac\alpha2)^4\sin(\frac\varphi2)^2\sin(\alpha)^2||\v_{\parallel\vec P}||_{L^2}^4
\\&+2\cos(\frac\alpha2)\cos(\frac\varphi2) \sin(\frac\alpha2)^3\sin(\frac\varphi2)\sin(\alpha)||\langle e^{\varphi\vec Q}\rangle_2||\,||\v_{\parallel\vec P}||_{L^2}^2,
\end{aligned} 
\end{equation}
remove the identical parts, identify $2\sin(\frac\alpha2)\cos(\frac\alpha2)=\sin(\alpha)$, divide both sides by $\sin(\frac\varphi2)\sin(\alpha)^2$, and get
\begin{equation}
\begin{aligned}\label{bew_b>a_absch}
\Leftrightarrow& \sin(\frac\varphi2)\cos(\frac\alpha2)^4||\v_{\parallel\vec P}||_{L^2}^4
-||\langle e^{\varphi\vec Q}\rangle_2|| \cos(\frac\varphi2)\cos(\frac\alpha2)^2||\v_{\parallel\vec P}||_{L^2}^2
\\&+\sin(\frac\varphi2)||\int\v_{\parallel\vec P}\v_{\perp\vec P}\d \x||^2
\\&\leq
\sin(\frac\alpha2)^4\sin(\frac\varphi2)||\v_{\parallel\vec P}||_{L^2}^4
+\cos(\frac\varphi2) \sin(\frac\alpha2)^2||\langle e^{\varphi\vec Q}\rangle_2||\,||\v_{\parallel\vec P}||_{L^2}^2.
 \\\Leftrightarrow&
 \sin(\frac\varphi2)\big((\cos(\frac\alpha2)^4-\sin(\frac\alpha2)^4)||\v_{\parallel\vec P}||_{L^2}^4
+||\int\v_{\parallel\vec P}\v_{\perp\vec P}\d \x||^2\big)
\\&\leq
\cos(\frac\varphi2) ||\langle e^{\varphi\vec Q}\rangle_2||( \sin(\frac\alpha2)^2+\cos(\frac\alpha2)^2)||\v_{\parallel\vec P}||_{L^2}^2
\\\Leftrightarrow&
 \sin(\frac\varphi2)(\cos(\alpha)||\v_{\parallel\vec P}||_{L^2}^4
+||\int\v_{\parallel\vec P}\v_{\perp\vec P}\d \x||^2)
\\&\leq
\cos(\frac\varphi2) ||\langle e^{\varphi\vec Q}\rangle_2||\,||\v_{\parallel\vec P}||_{L^2}^2.
\end{aligned} 
\end{equation}
In (\ref{bew_b>a_absch}) we identified $\cos(\frac\alpha2)^4-\sin(\frac\alpha2)^4=\cos(\alpha)$ and $\sin(\frac\alpha2)^2+\cos(\frac\alpha2)^2=1$. We further use (\ref{scphiq}) to replace $\cos(\alpha)$ by $(\cos(\varphi)-||\v_{\perp\vec P}||_{L^2}^2)||\v_{\parallel\vec P}||_{L^2}^{-2}$ and (\ref{bivphiq}) to replace $||\langle e^{\varphi\vec Q}\rangle_2||$ by $\sin(\varphi)$, which leads to
\begin{equation}
\begin{aligned}\label{beweisende}
\Leftrightarrow
&\sin(\frac\varphi2)\big((\cos(\varphi)-||\v_{\perp\vec P}||_{L^2}^2)||\v_{\parallel\vec P}||_{L^2}^2
+||\int\v_{\parallel\vec P}\v_{\perp\vec P}\d \x||^2\big)
\\&\leq
\cos(\frac\varphi2) \sin(\varphi)||\v_{\parallel\vec P}||_{L^2}^2
\\\Leftrightarrow
&\sin(\frac\varphi2)\big((2\cos(\frac\varphi2)^2-1-||\v_{\perp\vec P}||_{L^2}^2)||\v_{\parallel\vec P}||_{L^2}^2
+||\int\v_{\parallel\vec P}\v_{\perp\vec P}\d \x||^2\big)
\\&\leq
2\sin(\frac\varphi2)\cos(\frac\varphi2)^2||\v_{\parallel\vec P}||_{L^2}^2
\\\Leftrightarrow
&\sin(\frac\varphi2)\big(-||\v_{\parallel\vec P}||_{L^2}^2-||\v_{\parallel\vec P}||_{L^2}^2||\v_{\perp\vec P}||_{L^2}^2
+||\int\v_{\parallel\vec P}\v_{\perp\vec P}\d \x||^2\big)
\leq
0.
\end{aligned} 
\end{equation}
The Cauchy Schwartz inequality (CSI) of $L^2(\R)$ guarantees
\begin{equation}
\begin{aligned}\label{csi}
||\int\v_{\parallel\vec P}(\x)\v_{\perp\vec P}(\x)\d\x||^2
\leq&
(\int ||\v_{\parallel\vec P}(\x)\v_{\perp\vec P}(\x)||\d\x)^2
\\=&
(\int ||\v_{\parallel\vec P}(\x)||\,||\v_{\perp\vec P}(\x)||\d\x)^2
\\=&\langle||\v_{\parallel\vec P}(\x)||,||\v_{\perp\vec P}(\x)||\rangle_{L^2}^2
\\\overset{\text{CSI}}\leq&
||\v_{\parallel\vec P}||_{L^2}^2||\v_{\perp\vec P}||_{L^2}^2,
\end{aligned} 
\end{equation}
so we know, that the part in the brackets in the last line of (\ref{beweisende}) is always negative. The sine is always positive for $\varphi\in[0,\pi]$. Therefore in the shape of (\ref{beweisende}) it is easy to recognize that the inequality $\beta\leq\alpha$ is generally fulfilled.%\smartqed\qed
\end{proof}
We want to apply Lemma \ref{l:outer_prod} repeatedly and construct a series of decreasing angles, describing the remaining misalignment of the vector fields. The next theorem shows that the misalignment vanishes by iteration.
\begin{thm}\label{t:conv2}
For a square integrable vector field $\v\in L^2(\R^m,\R^{3,0}\subset\clifford{3,0})$ let $\beta:[0,\pi)\to[0,\pi)$ be a function defined by $\beta(\alpha)=2\arg(e^{\frac\alpha2 \vec P}e^{\frac\varphi2 \vec Q})$ with 
\begin{equation}
\begin{aligned}
e^{\varphi\vec Q}=\frac{(\operatorname{R} _{P,\alpha}(\v)\star\v)(0)}{{||\operatorname{R} _{P,\alpha}(\v)||_{L^2}||\v||_{L^2}}}.
\end{aligned} 
\end{equation}
Then the series $\alpha_0=\alpha,\alpha_{n+1}=\beta(\alpha_{n})$ converges to zero for all $\alpha\in[0,\pi)$.
\end{thm}
\begin{proof}
If $\alpha=0$ or $||\v_{\parallel\vec P}||_{L^2}=0$ the series is trivial because $\operatorname{R} _{P,\alpha}(\v)=\v$ almost everywhere. From now on let $\alpha\neq0$ and $||\v_{\parallel\vec P}||_{L^2}\neq0$. Lemma \ref{l:outer_prod} shows that the magnitudes of the series are monotonically decreasing. Since they are bound from below by zero the series is convergent. We denote the limit by $a=\lim_{n\to\infty}\alpha_n$. The function $\beta(\alpha)$ is continuous, so we can swap it and the limit
\begin{equation}
\begin{aligned}\label{beta=alpha}
a=\lim_{n\to\infty}\alpha_{n+1}=\lim_{n\to\infty}\beta(\alpha_n)=\beta(\lim_{n\to\infty} \alpha_n)=\beta(a).
\end{aligned} 
\end{equation}
This equality is the sharp case of the inequality (\ref{beta<alpha}). For $\alpha\neq0$ also $\varphi\neq0$, because from (\ref{scphiq}), $||\v_{\parallel\vec P}||_{L^2}^2+||\v_{\perp\vec P}||_{L^2}^2=1$, and $||\v_{\parallel\vec P}||_{L^2}\neq0$ we get
\begin{equation}
\begin{aligned}\label{alpha0phi0}
\varphi=0
\Leftrightarrow
1=\cos(\varphi)=\cos(\alpha)||\v_{\parallel\vec P}||_{L^2}^2+||\v_{\perp\vec P}||_{L^2}^2
\Leftrightarrow
\cos(\alpha)=1
\Leftrightarrow
\alpha=0
\end{aligned} 
\end{equation} 
Therefore the transformative steps that lead from (\ref{beta<alpha}) to (\ref{beweisende}) in the proof of Lemma \ref{l:outer_prod} are biconditional. So analogously to these steps (\ref{beta=alpha}) is equivalent to
\begin{equation}
\begin{aligned}
&\sin(\frac\varphi2)(-||\v_{\parallel\vec P}||_{L^2}^2-||\v_{\parallel\vec P}||_{L^2}^2||\v_{\perp\vec P}||_{L^2}^2
+||\int\v_{\parallel\vec P}\v_{\perp\vec P}\d \x||^2)
=
0
\end{aligned} 
\end{equation}
The part in the brackets is strictly negative, because of the Cauchy Schwartz inequality (\ref{csi}) and $||\v_{\parallel\vec P}||_{L^2}\neq0$, therefore equality can only occur for $\sin(\frac\varphi2)=0$, which means $\varphi=0$. Like in (\ref{alpha0phi0}) this leaves $a=0$ as the only possible limit.%\smartqed\qed
\end{proof}
%----------------------------------------------------------------------------------------------------------------------------
\section{Algorithm and Experiments}
%----------------------------------------------------------------------------------------------------------------------------
Motivated by Theorem \ref{t:conv2} we present Algorithm \ref{alg1} for the iterative detection of outer rotations of vector fields using geometric cross correlation. It has been designed with attention to the efficient use of memory and to handle possible exceptions. In the case of $\alpha=\pi$ the correlation might be real valued and can not be distinguished from the cases where no rotation is necessary. To fix this exception we suggest an artificial disturbance after the first step of the algorithm, compare Line 7 in Algorithm \ref{alg1}. If it is not real the new misalignment will be smaller, because any misalignment is smaller than $\pi$. 
In all other cases Theorem \ref{t:conv2} guarantees the convergence. Our results also apply to the geometric product of the vector fields at any position $\x\in\R^m,\v_{\parallel\vec P}(\x)\neq0$, but we prefer the geometric correlation because of its robustness.
\begin{cor}
Algorithm \ref{alg1} returns the correct rotational misalignment for any three-dimensional linear vector field and its copy generated from an arbitrary outer rotation.
\end{cor}
%
%\begin{small}
\begin{algorithm}
\caption{Detection of outer misalignment of vector fields in 3D}
\label{alg1}
\begin{algorithmic}[1]
\REQUIRE vector field: $\v(\x)$, rotated pattern: $\u(\x)$, desired accuracy: $\varepsilon>0$,
\STATE $\varphi=\pi,\alpha=0, \vec P=\e_{12},iter=0$,
\WHILE{$\varphi>\varepsilon$}
  \STATE $iter++$,
  \STATE $Cor=(\u(\x)\star \v(\x))(0)$,
  \STATE $\varphi=\arg(Cor)$,
  \STATE $\vec Q=\langle Cor\rangle_2|\langle Cor\rangle_2|^{-1}$,
  \IF {$iter=1$ and $\varphi=0$}
    \STATE $\varphi=\pi/4$,
    \STATE $\vec Q=\e_{12}$,
  \ENDIF
  \STATE $\u(\x)=e^{-\frac\varphi2\vec Q}\u(\x)e^{\frac\varphi2\vec Q}$,
  \STATE $\alpha^\prime=2\arg(e^{\frac\alpha2\vec P}e^{\frac\varphi2\vec Q})$,
  \STATE $\vec P=\langle e^{\frac\alpha2\vec P}e^{\frac\varphi2\vec Q}\rangle_2|\langle e^{\frac\alpha2\vec P}e^{\frac\varphi2\vec Q}\rangle_2|^{-1}$,
  \STATE $\alpha=\alpha^\prime$,
 \ENDWHILE
\ENSURE angle: $\alpha$, plane: $\vec P$, corrected pattern: $\u(\x)$, iterations needed: $iter$.
\end{algorithmic}
\end{algorithm}
%\end{small}
%
%----------------------------------------------------------------------------------------------------------------------------
%\section{Experiments}
%----------------------------------------------------------------------------------------------------------------------------
\par
We practically tested Algorithm \ref{alg1} applying it to continuous, linear vector fields $\R^3\to\clifford{3,0}$, that vanish outside the unit square. The vector fields were determined by nine random coefficients with magnitude not bigger than one. The plane and the angle $\alpha\in[0,\pi]$ of the outer rotation were also chosen randomly. The average results of 1000 applications can be found in Table \ref{tab:1}. The error was measured from the sum of the squared differences of the determined and the given coefficients. The experiments showed that 
%unsuccessful runs (i.e. the sum of the squared differences of determined and given coefficients exceeds 0.04) 
high numbers of necessary iterations are much more likely to happen for angles with high magnitude and that the average error decreases linearly with the demanded accuracy. But most importantly we observed that Algorithm \ref{alg1} converged in all linear cases, just as the theory suggested. The hypercomplex correlation method of Moxey et al. in \cite{MSE03} can be interpreted as one step of Algorithm \ref{alg1}, i.e., to terminate without iteration. The last row of Table \ref{tab:1} shows how the application of more iterations increases the accuracy. All tests we performed on a computer with two Intel Xeon E5620 processors and 32GB RAM. Even though parallelization of the computation can easily be accomplished, the given results refer to the linear computation.
\par
\begin{table}[ht]
\begin{center}
\begin{tabular}{|l|r|r|r|r|r|} 
\hline
number of iterations	&0	&1	& 10	& 100	& 1000		\rule [-1.2mm]{0mm}{5mm}\\\hline
absolute error		&3.925	&0.915	&0.039	& $10^{-4}$& $10^{-12}$	\rule [-1.2mm]{0mm}{5mm}\\\hline
error per coefficient	&0.436	&0.102	&0.004	& $10^{-5}$&  $10^{-13}$\rule [-1.2mm]{0mm}{5mm}\\\hline
duration in seconds	&0	&0	&$10^{-6}$&0.0004& 0.0031	\rule [-1.2mm]{0mm}{5mm}\\\hline
\end{tabular}%\label{tab:g2}
\caption{Results of Algorithm \ref{alg1} applied to continuous linear vector fields depending on the number of iteration steps.\label{tab:1}}
\end{center}
\end{table}

% \begin{table}[ht]
% \begin{center}
% \begin{tabular}{|l|r|r|r|r|r|} 
% \hline
% determined accuracy $eps$	&0.1	& 0.01	& 0.001	& 0.0001	& 0.00001	 \rule [-1.2mm]{0mm}{5mm}\\\hline
% %relative number of fails in \%	&28.04	& 0.13	& 0.002	& 0	&0 	\rule [-1.2mm]{0mm}{5mm}\\
% average error			&0.17	& 0.02	& 0.002	& 0.0002	& 0.00002	\rule [-1.2mm]{0mm}{5mm}\\\hline
% maximal error			&2.49	& 1.29	& 0.47	& 0.12		& 0.001		\rule [-1.2mm]{0mm}{5mm}\\\hline
% average number of iterations	&4.23	& 11.76	& 21.44	& 31.78		& 42.26		\rule [-1.2mm]{0mm}{5mm}\\\hline
% \end{tabular}%\label{tab:g2}
% \caption{
% Results of Algorithm \ref{alg1} depending on the required accuracy.\label{tab:1}}
% \end{center}
% \end{table}

Further, we tested the algorithm in a more practical and applied way. In contrast to the previous experiments, we worked with discrete data that does not obey any linearity properties. In this case we can not calculate the correlation analytically. We have to approximate it.

%We practically tested Algorithm \ref{alg1} applying it to discrete images with rotationally distorted color spaces.
Representatively, we chose a picture of Leipzig University with 885 times 622 pixels. The components of the vector field are the red, green and blue values of the image, which are each scaled to $[-0.5,0.5]$. The image features a balanced distribution of the three channels, the average is $(0.14,0.15,0.15)$. We applied a rotation in color space about the red axis by an angle of $1.7$ and 
let Algorithm \ref{alg1} detect it. The visual results after 1, 10, 100, and 1000 steps of iteration can be seen in Figure \ref{f:4} and the computational errors in Table \ref{tab:2}. The absolute error is the sum of the norms of the differences

\begin{figure}[ht]
\centering
\subfigure[Original image]{\includegraphics[width=0.48\textwidth]{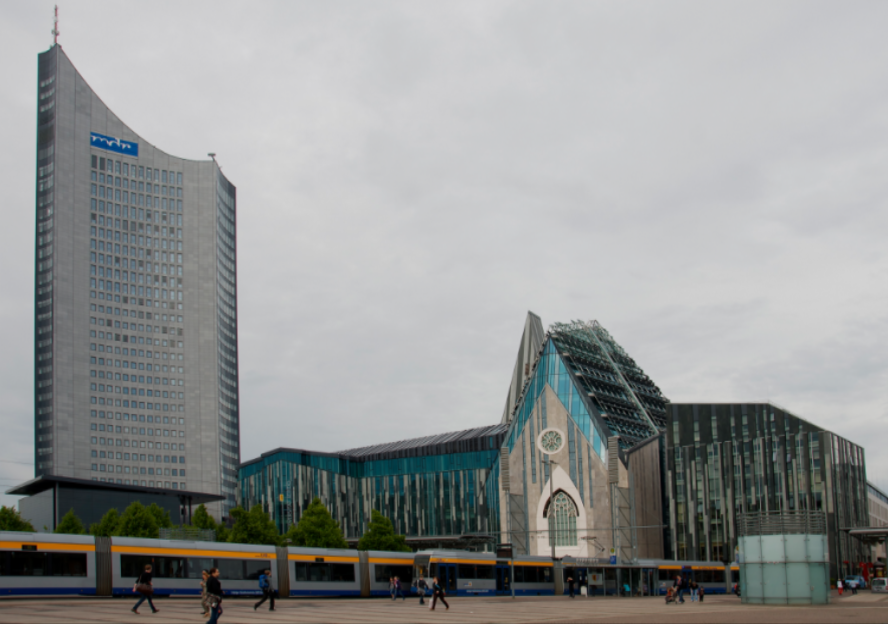}}\quad
\subfigure[Distorted image]{\includegraphics[width=0.48\textwidth]{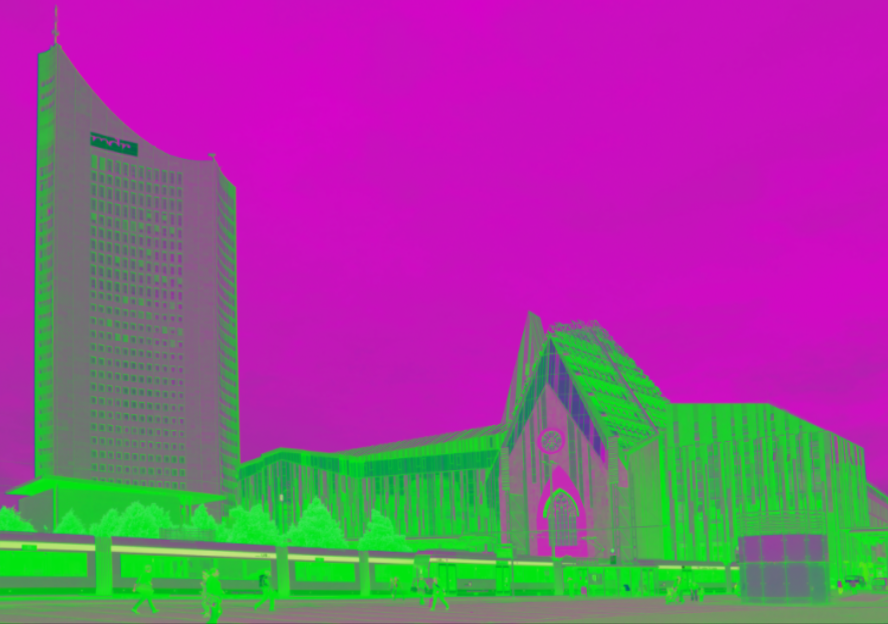}}\quad
\subfigure[Restored image after 1 step]{\includegraphics[width=0.48\textwidth]{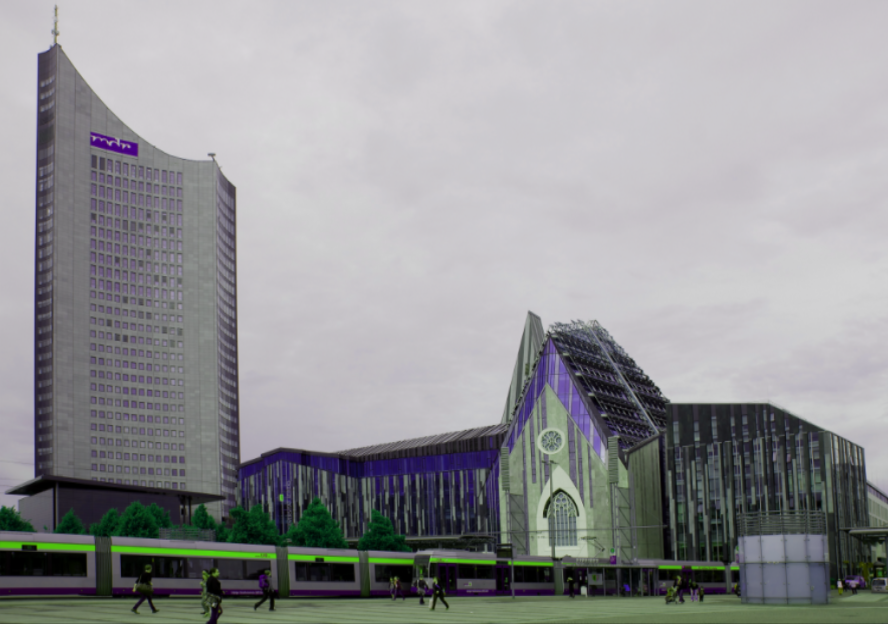}}\quad
\subfigure[Restored image after 10 steps]{\includegraphics[width=0.48\textwidth]{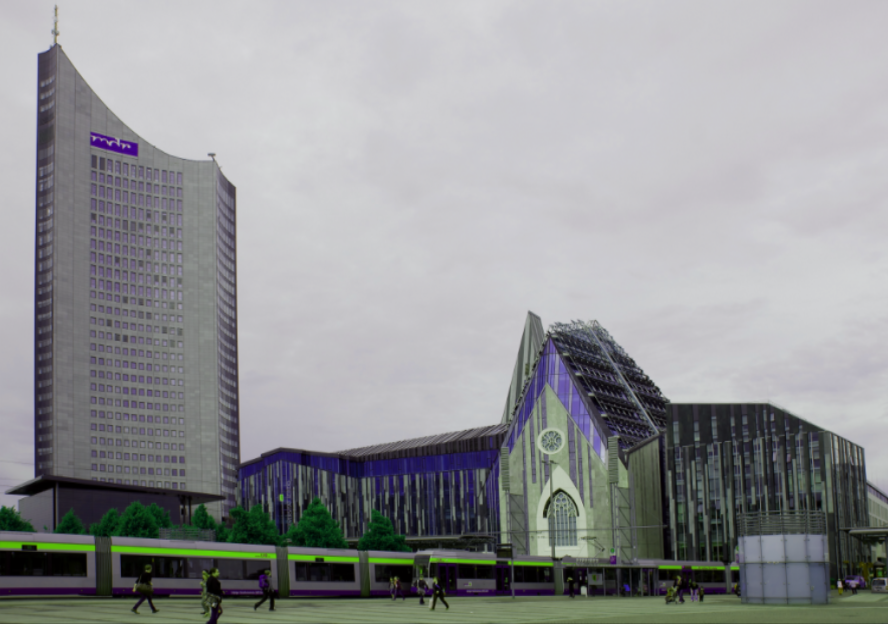}}
\subfigure[Restored image after 100 steps]{\includegraphics[width=0.48\textwidth]{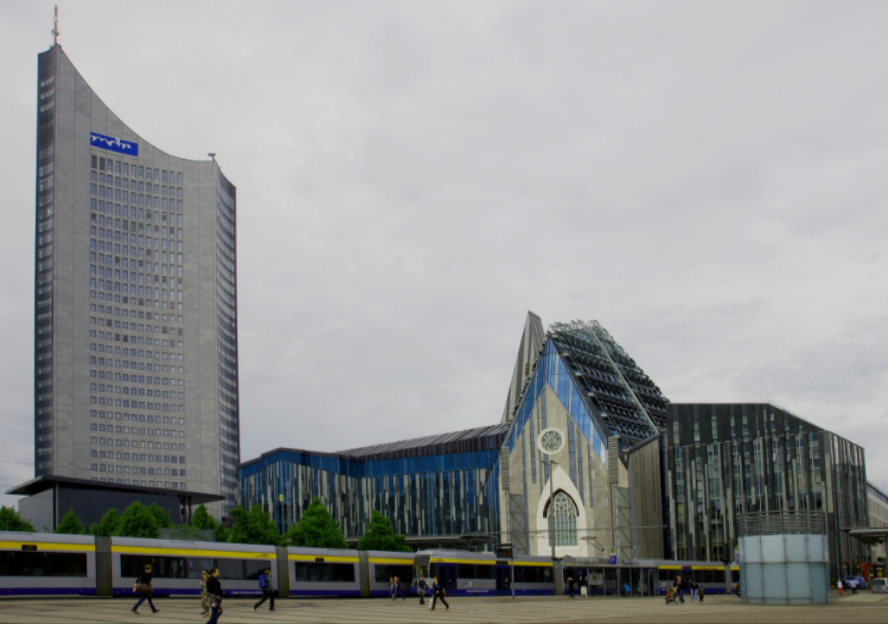}}\quad
\subfigure[Restored image after 1000 steps]{\includegraphics[width=0.48\textwidth]{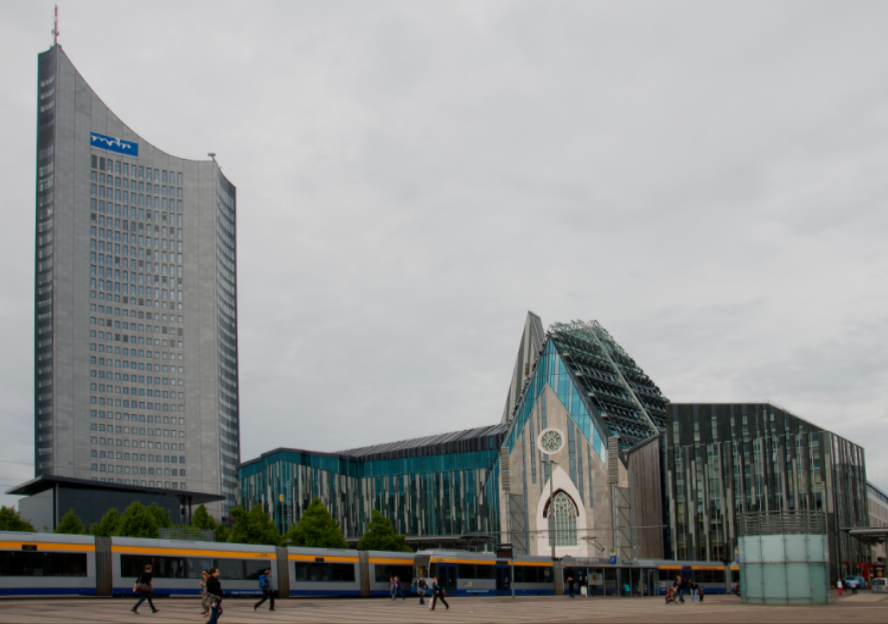}}
\caption{Effect of Algorithm \ref{alg1} applied to the distorted color image of Leipzig University after several steps of iteration.\label{f:4}} 
\end{figure}

\begin{table}[ht]
\begin{center}
\begin{tabular}{|l|r|r|r|r|r|} 
\hline
number of iterations	&0	&1	& 10	& 100	& 1000		\rule [-1.2mm]{0mm}{5mm}\\\hline
absolute error		&258136	&2428	& 2068	& 309	& $10^{-7}$	\rule [-1.2mm]{0mm}{5mm}\\\hline
error per pixel		&0.4676	&0.0044	& 0.0037& 0.0005& $10^{-13}$	\rule [-1.2mm]{0mm}{5mm}\\\hline
duration in seconds	&3.1	&3.5	& 7.3	& 44	& 413		\rule [-1.2mm]{0mm}{5mm}\\\hline
\end{tabular}%\label{tab:g2}
\caption{
Results of Algorithm \ref{alg1} applied to the distorted color image of Leipzig University depending on the number of iteration steps.\label{tab:2}}
\end{center}
\end{table}
We also tested Algorithm \ref{alg1} on other images with differing sized and color distributions and different axes and angles of rotational distortion. As it turned out, it hardly influenced the outcome of the algorithm. The chosen picture is a good representative of the overall properties of the algorithm.
\par
Our experiments showed, that the algorithm always converged for any color image and any rotational misalignment. Please note that the discrete case is not covered by Theorem \ref{t:conv2}, so this result was not guaranteed. But they also revealed the high number of necessary iterations to approximate the original image in a way that the difference can not be told by the human eye. This high computational effort makes Algorithm \ref{alg1} very inefficient and of no practical relevance.
%--------------------------------------------------------------------------------------------------------------
\section{Acceleration}
%--------------------------------------------------------------------------------------------------------------
An interesting observation is the big jump into the right direction the algorithm performs during its first step. That explains the approach of Moxey et al. \cite{MSE03} and its success. Further the result of the first step is almost the same for every initial misalignment, no matter in which plane the color space had been rotated in the first place. An explanation for these phenomenons can be given as follows. The algorithm immediately finds a plane in which increasing rotational misalignment leads to small growing differences in the image and continues to move along this plane. This effect can be seen in our example image. Most of the pixels are very close to their original color after the first step already. The remaining distortion can only be observed in the yellow stipe of the tram and the blue windows. As a result the speed of convergence in this plane is rather small because all the pixels that have achieved their original color already will decrease the  angle of rotation with their real valued 
contribution to the correlation. 
\par
This observation gives rise to an idea to accelerate the algorithm. We can assume the plane that is detected during the second step of the algorithm to be very close to the correct plane of the remaining misalignment. Therefore the correlation of only the parallel components will return almost the correct angle as in the 2D case (\ref{outerprod}). Iterative Application of this idea leads to Algorithm \ref{alg2}. 
\par
\begin{algorithm}
\caption{Fast detection of outer misalignment of vector fields in 3D}
\label{alg2}
\begin{algorithmic}[1]
\REQUIRE vector field: $\v(\x)$, rotated pattern: $\u(\x)$, desired accuracy: $\varepsilon>0$,
\STATE $\varphi=\pi,\alpha=0, \vec P=\e_{12},iter=0$,
\WHILE{$\varphi>\varepsilon$}
  \STATE $iter++$,
  \STATE $Cor=(\u(\x)\star \v(\x))(0)$,
  \STATE $\varphi=\arg(Cor)$,
  \STATE $\vec Q=\langle Cor\rangle_2|\langle Cor\rangle_2|^{-1}$,
  \IF {$iter=1$ and $\varphi=0$}
    \STATE $\varphi=\pi/4$,
    \STATE $\vec Q=\e_{12}$,
  \ENDIF
  \STATE $\u(\x)=e^{-\frac\varphi2\vec Q}\u(\x)e^{\frac\varphi2\vec Q}$,
  \STATE $Cor=(\u(\x)\star \v(\x))(0)$,
  \STATE $\vec Q=\langle Cor\rangle_2|\langle Cor\rangle_2|^{-1}$,
  \STATE $Cor=(\u_{\parallel \vec Q}(\x)\star \v_{\parallel \vec Q}(\x))(0)$,
  \STATE $\varphi=\arg(Cor)$,
  \STATE $\u(\x)=e^{-\frac\varphi2\vec Q}\u(\x)e^{\frac\varphi2\vec Q}$,
  \STATE $\alpha^\prime=2\arg(e^{\frac\alpha2\vec P}e^{\frac\varphi2\vec Q})$,
  \STATE $\vec P=\langle e^{\frac\alpha2\vec P}e^{\frac\varphi2\vec Q}\rangle_2|\langle e^{\frac\alpha2\vec P}e^{\frac\varphi2\vec Q}\rangle_2|^{-1}$,
  \STATE $\alpha=\alpha^\prime$,
 \ENDWHILE
\ENSURE angle: $\alpha$, plane: $\vec P$, corrected pattern: $\u(\x)$, iterations needed: $iter$.
\end{algorithmic}
\end{algorithm}
We implemented Algorithm \ref{alg2} and tested it for some example images. The result was astonishing. We could generally observe convergence and the speed was about 100 times the speed of Algorithm \ref{alg1}. The errors and calculation times of Algorithm \ref{alg2} can be found in Table \ref{tab:3}.
\par
\begin{table}[ht]
\begin{center}
\begin{tabular}{|l|r|r|r|r|r|} 
\hline
number of iterations	&0	&1	&2	& 3	& 4		\rule [-1.2mm]{0mm}{5mm}\\\hline
absolute error		&258136	&2337	&17.595	& 0.0966& 0.0005	\rule [-1.2mm]{0mm}{5mm}\\\hline
error per pixel		&0.4676	&0.0042	&$10^{-5}$&$10^{-7}$&$10^{-9}$	\rule [-1.2mm]{0mm}{5mm}\\\hline
duration in seconds	&3.1	&4.7	&6.0	&7.4	&8.9		\rule [-1.2mm]{0mm}{5mm}\\\hline
\end{tabular}%\label{tab:g2}
\caption{
Results of Algorithm \ref{alg2} applied to the distorted color image of Leipzig University depending on the number of iteration steps.\label{tab:3}}
\end{center}
\end{table}
A visualization of the approximation using the accelerated algorithm can be found in Figure \ref{f:5}. Already after the second step no difference to the original image can be seen any more. That is why we did not print the visual results of the following steps. 
\par
\begin{figure}[ht]
\centering
\subfigure[Original image]{\includegraphics[width=0.48\textwidth]{uni_orig}}\quad
\subfigure[Distorted image]{\includegraphics[width=0.48\textwidth]{uni0}}\quad
\subfigure[Restored image after 1 step]{\includegraphics[width=0.48\textwidth]{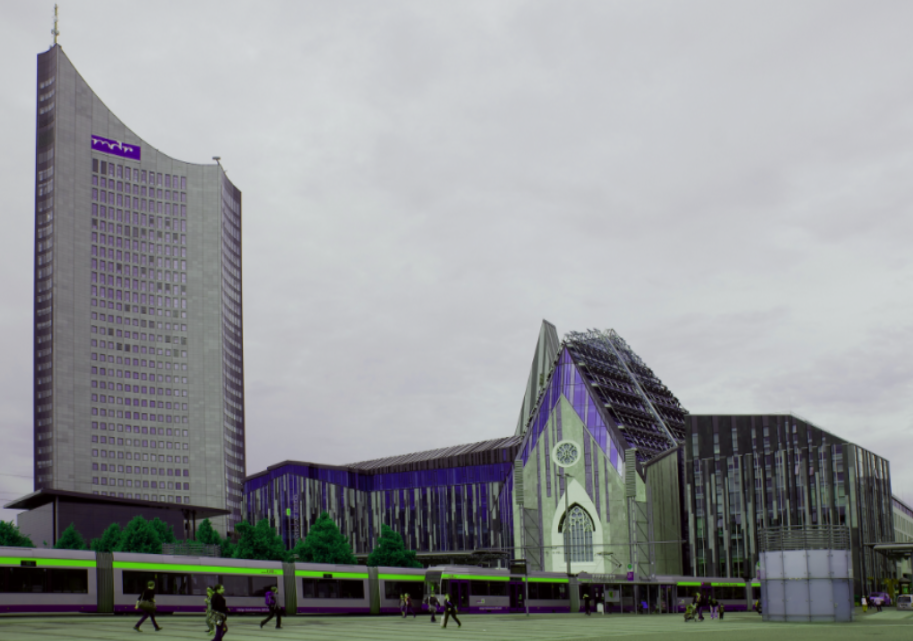}}\quad
\subfigure[Restored image after 2 steps]{\includegraphics[width=0.48\textwidth]{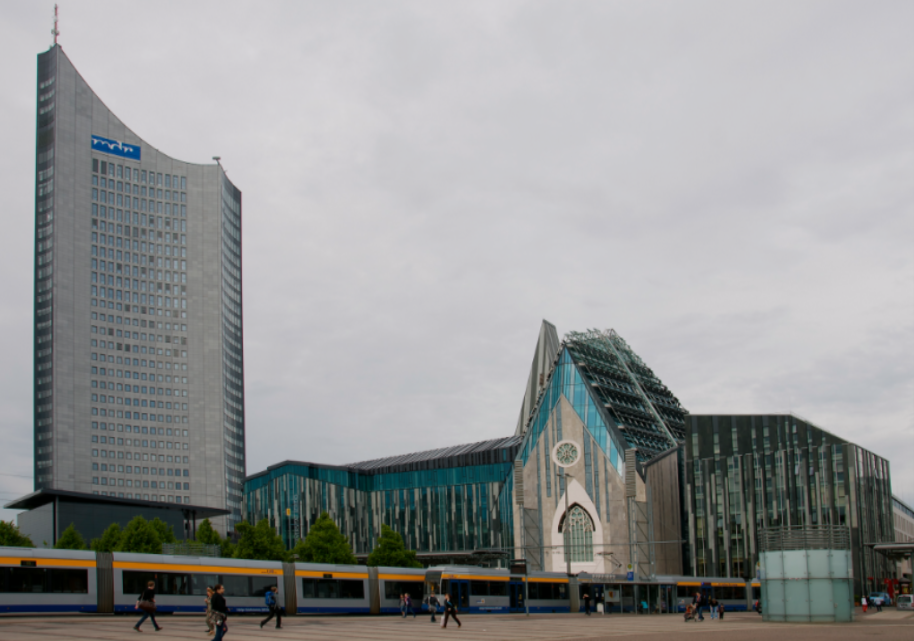}}
\caption{Effect of Algorithm \ref{alg2} applied to the distorted color image of Leipzig University after one and two steps of iteration.\label{f:5}} 
\end{figure}
Please note that the convergence of Algorithm \ref{alg2} has only been observed experimentally but not yet been proved. The qualities of the plane chosen by the algorithm deserve fundamental theoretical analysis in the future.
% \begin{table}
% \caption{Results of Algorithm \ref{alg1} depending on the required accuracy}
% \label{tab:1}       % Give a unique label
% %
% % Follow this input for your own table layout
% %
% \begin{tabular}{p{4.05cm}p{1.4cm}p{1.4cm}p{1.4cm}p{1.4cm}p{1.4cm}}
% \hline\noalign{\smallskip}
% determined accuracy $eps$	&0.1	& 0.01	& 0.001	& 0.0001	& 0.00001\\
% \noalign{\smallskip}\svhline\noalign{\smallskip}
% average error			&0.17	& 0.02	& 0.002	& 0.0002	& 0.00002\\
% %\noalign{\smallskip}\svhline\noalign{\smallskip}
% maximal error			&2.49	& 1.29	& 0.47	& 0.12		& 0.001\\
% %\noalign{\smallskip}\svhline\noalign{\smallskip}
% average number of iterations	&4.23	& 11.76	& 21.44	& 31.78		& 42.26\\
% \noalign{\smallskip}\hline\noalign{\smallskip}
% \end{tabular}
% \end{table}
%----------------------------------------------------------------------------------------------------------------------------
\section{Conclusions and Outlook}
%----------------------------------------------------------------------------------------------------------------------------
The geometric cross correlation of two vector fields is scalar and bivector valued. Moxey et al. \cite{MSE03} realized that this rotor yields an approximation of the outer rotational misalignment of vector fields. We analyzed that the quality of this approximation depends on the parallel and the orthogonal parts of the fields and proved in Lemma \ref{l:outer_prod} that the application of this rotor to the outer rotated copy of any vector field never increases the misalignment to the original field. In Theorem \ref{t:conv2} we refined this fact and showed that iterative application completely erases the misalignment of the rotationally misaligned vector fields. 
\par
We presented Algorithm \ref{alg1}, which additionally contains exception handling, and experimentally confirmed our theoretical findings. Our experiments showed general convergence even in the case of discrete fields, but a low rate of convergence. From experimental observation we deduced the idea for Algorithm \ref{alg2} and practically showed its superior performance.
\par
All in all we consider the convergence of the iterative geometric correlation as in Algorithm \ref{alg1} to be of no practical relevance, but an interesting feature of the geometric product. Its acceleration, as shown in Algorithm \ref{alg2}, is worth further studies.
\par
We currently analyze the application of this approach to total rotations in \cite{BSH12b}. Further we examine the properties of the plane the iterative correlation suggests from the second step on. This may be the key to proving the convergence of Algorithm \ref{alg2}. 
%Bujack, Scheuermann and Hitzer: ``Detection of Total Rotations on linear 2D-Vector Fields with Iterative Geometric Correlation'', accepted at ICNPAA 2012. 
In our future work we want to further accelerate the algorithm by means of a fast Fourier transform and a geometric convolution theorem \cite{BSH12c}. 
Another promising idea to pursue is the development of a customized convolution Clifford \cite{BBSS13} that is able to detect the misalignment without iteration.
Practically it might be of interest how the algorithm is able to deal with vector fields, that are disturbed or differ by more complex transformations and if it is able to minimize the squared differences if the fields are not equal after rotation but only similar.

\section*{Acknowledgements}
The authors would like to thank the FAnToM development group from Leipzig University for providing the environment and the data for the visualization of the presented work, especially Stefan Koch and Mario Hlawitschka. This work was partially supported by the European Social Fund (Application No. 100098251).

 %\addcontentsline{toc}{section}{References}
  \bibliographystyle{plain} 
  \bibliography{Literaturverzeichnis}

\end{document}